\documentclass[runningheads]{llncs}
\usepackage{graphicx}
\usepackage{color}
\usepackage{epsfig}
\usepackage{times}
\usepackage{amsmath,amssymb} 

\usepackage{epigraph}
\usepackage{algorithm}
\usepackage{algorithmic}
\usepackage{url}
\usepackage{enumitem}
\usepackage{color}
\usepackage{soul}
\usepackage{subfigure}
\usepackage{mdframed}
\usepackage{enumitem}
\usepackage{bm}
\usepackage{braket}
\usepackage{wasysym}
\usepackage{marvosym}

\usepackage{tikz}
\usetikzlibrary{quantikz}

\usepackage[breaklinks=true,bookmarks=false]{hyperref}

\newcommand{\bx}{\mathbf{x}}

\newcommand{\bz}{\mathbf{z}}

\newcommand{\be}{\mathbf{e}}

\newcommand{\bs}{\mathbf{s}}

\newcommand{\bzero}{\mathbf{0}}

\newcommand{\bp}{\mathbf{p}}

\newcommand{\bq}{\mathbf{q}}

\newcommand{\bu}{\mathbf{u}}
\newcommand{\bv}{\mathbf{v}}
\newcommand{\bt}{\mathbf{t}}

\newcommand{\cM}{\mathcal{M}}
\newcommand{\cB}{\mathcal{B}}

\newcommand{\bH}{\mathbf{H}}

\newcommand{\cI}{\mathcal{I}}

\newcommand{\cD}{\mathcal{D}}

\newcommand{\cO}{\mathcal{O}}
\newcommand{\bP}{\mathbf{P}}

\newcommand{\cC}{\mathcal{C}}

\newcommand{\cA}{\mathcal{A}}

\newcommand{\sfZ}{\mathsf{Z}}

\def\argmax{\operatorname*{argmax\,}}

\begin{document}
\pagestyle{headings}
\mainmatter

\def\ACCV20SubNumber{139}  

\title{Quantum Robust Fitting} 
\titlerunning{Quantum Robust Fitting}
%
\author{Tat-Jun Chin\inst{1}\orcidID{0000-0003-2423-9342} \and
David Suter\inst{2}\orcidID{0000-0001-6306-3023} \and \\
Shin-Fang Ch'ng\inst{1}\orcidID{0000-0003-1092-8921} \and
James Quach\inst{3}\orcidID{0000-0002-3619-2505} }
\authorrunning{T.-J. Chin et al.}
%
\institute{School of Computer Science, The University of Adelaide \and
School of Computing and Security, Edith Cowan University \and
School of Physical Sciences, The University of Adelaide}

\maketitle

\begin{abstract}
Many computer vision applications need to recover structure from imperfect measurements of the real world. The task is often solved by robustly fitting a geometric model onto noisy and outlier-contaminated data. However, recent theoretical analyses indicate that many commonly used formulations of robust fitting in computer vision are not amenable to tractable solution and approximation. In this paper, we explore the usage of quantum computers for robust fitting. To do so, we examine and establish the practical usefulness of a robust fitting formulation inspired by the analysis of monotone Boolean functions. We then investigate a quantum algorithm to solve the formulation and analyse the computational speed-up possible over the classical algorithm. Our work thus proposes one of the first quantum treatments of robust fitting for computer vision.
\end{abstract}

\section{Introduction}

Curve fitting is vital to many computer vision capabilities~\cite{hartnett18}. We focus on the special case of ``geometric" curve fitting~\cite{kanatani16}, where the curves of interest derive from the fundamental constraints that govern image formation and the physical motions of objects in the scene. Geometric curve fitting is conducted on visual data that is usually contaminated by outliers, thus necessitating robust fitting.

To begin, let $\cM$ be a geometric model parametrised by a vector $\bx \in \mathbb{R}^d$. For now, we will keep $\cM$ generic; specific examples will be given later. Our aim is to fit $\cM$ onto $N$ data points $\cD = \{ \bp_i \}^{N}_{i=1}$, i.e., estimate $\bx$ such that $\cM$ describes $\cD$ well. To this end, we employ a residual function
\begin{align}
r_i(\bx)
\end{align}
which gives the nonnegative error incurred on the $i$-th datum $\bp_i$ by the instance of $\cM$ that is defined by $\bx$. Ideally we would like to find an $\bx$ such that $r_i(\bx)$ is small for \emph{all} $i$.

However, if $\cD$ contains outliers, there are no $\bx$ where all $r_i(\bx)$ can be simultaneously small. To deal with outliers, computer vision practitioners often maximise the consensus
\begin{align}\label{eq:consensus}
\Psi(\bx) = \sum^{N}_{i=1} \mathbb{I}( r_i(\bx) \le \epsilon )
\end{align}
of $\bx$, where $\epsilon$ is a given inlier threshold, and $\mathbb{I}$ is the indicator function that returns $1$ if the input predicate is true and $0$ otherwise. Intuitively, $\Psi(\bx)$ counts the number of points that agree with $\bx$ up to threshold $\epsilon$, which is a robust criterion since points that disagree with $\bx$ (the outliers) are ignored~\cite{chin17}. The maximiser $\bx^\ast$, called the maximum consensus estimate, agrees with the most number of points.

To maximise consensus, computer vision practitioners often rely on randomised sampling techniques, i.e., RANSAC~\cite{fischler81} and its variants~\cite{raguram13}. However, random sampling cannot guarantee finding $\bx^\ast$ or even a satisfactory alternative. In fact, recent analysis~\cite{chin18} indicates that there are no efficient algorithms that can find $\bx^\ast$ or bounded-error approximations thereof. In the absence of algorithms with strong guarantees, practitioners can only rely on random sampling methods~\cite{fischler81,raguram13} with supporting heuristics to increase the chances of finding good solutions.

Robust fitting is in fact intractable in general. Beyond maximum consensus, the fundamental hardness of robust criteria which originated in the statistics community (e.g., least median squares, least trimmed squares) have also been established~\cite{bernholt06}. Analysis on robust objectives (e.g., minimally trimmed squares) used in robotics~\cite{tzoumas19} also point to the intractability and inapproximability of robust fitting.

In this paper, we explore a robust fitting approach based on ``influence" as a measure of outlyingness recently introduced by Suter et al.~\cite{suter2020monotone}. Specifically, we will establish
\begin{itemize}
	\item The practical usefulness of the technique;
	\item A probabilistically convergent classical algorithm; and
	\item A quantum algorithm to speed up the classical method, thus realising quantum robust fitting.
\end{itemize}

\subsection{Are all quantum computers the ``same"?}

Before delving into the details, it would be useful to paint a broad picture of quantum computing due to the unfamiliarity of the general computer vision audience to the topic.

At the moment, there are no practical quantum computers, although there are several competing technologies under intensive research to realise quantum computers. The approaches can be broadly classified into ``analog" and ``digital" quantum computers. In the former type, adiabatic quantum computers (AQC) is a notable example. In the latter type, (universal) gate quantum computers (GQC) is the main subject of research, in part due to its theoretically proven capability to factorise integers in polynomial time (i.e., Shor's algorithm). Our work here is developed under the GQC framework.

There has been recent work to solve computer vision problems using quantum computers, in particular~\cite{neven08,nguyen19,golyanik20}. However, these have been developed under the AQC framework, hence, the algorithms are unlikely to be transferrable easily to our setting. Moreover, they were not aimed at robust fitting, which is our problem of interest.


\section{Preliminaries}\label{sec:prelim}

Henceforth, we will refer to a data point $\bp_i$ via its index $i$. Thus, the overall set of data $\cD$ is equivalent to $\{1,\dots,N \}$ and subsets thereof are $\cC \subseteq \cD = \{1,\dots,N\}$.

We restrict ourselves to residuals $r_i(\bx)$ that are quasiconvex~\cite{kahl08} (note that this does not reduce the hardness of maximum consensus~\cite{chin18}). Formally, if the set
\begin{align}
\{ \bx \in \mathbb{R}^d \mid r_i(\bx) \le \alpha \}
\end{align}
is convex for all $\alpha \ge 0$, then $r_i(\bx)$ is quasiconvex. It will be useful to consider the minimax problem
\begin{align}\label{eq:minimax}
g(\cC) = \min_{\bx \in \mathbb{R}^d}~\max_{i \in \cC }~r_i(\bx),
\end{align}
where $g(\cC)$ is the minimised maximum residual for the points in the subset $\cC$. If $r_i(\bx)$ is quasiconvex then~\eqref{eq:minimax} is tractable in general~\cite{eppstein05}, and $g(\cC)$ is monotonic, viz.,
\begin{align}
\cB \subseteq \cC \subseteq \cD \implies g(\cB) \le g(\cC) \le g(\cD).
\end{align}
Chin et al.~\cite{chin15,chin18} exploited the above properties to develop a fixed parameter tractable algorithm for maximum consensus, which scales exponentially with the outlier count.

A subset $\cI \subseteq \cD$ is a consensus set if there exists $\bx\in\mathbb{R}^d$ such that $r_i(\bx) \le \epsilon$ for all $i \in \cI$. Intuitively, $\cI$ contains points that can be fitted within error $\epsilon$. In other words
\begin{align}
g(\cI) \le \epsilon
\end{align}
if $\cI$ is a consensus set. The set of all consensus sets is thus
\begin{align}
\mathbb{F} = \left\{ \cI \subseteq \cD \mid g(\cI) \le \epsilon \right\}.
\end{align}
The consensus maximisation problem can be restated as
\begin{align}\label{eq:maxcon}
\cI^\ast = \underset{\cI \in \mathbb{F} }{\argmax}~|\cI|,
\end{align}
where $\cI^\ast$ is the maximum consensus set. The maximum consensus estimate $\bx^\ast$ is a ``witness" of $\cI^\ast$, i.e., $r_i(\bx^\ast) \le \epsilon$ for all $i \in \cI^\ast$, and $|\cI^\ast| = \Psi(\bx^\ast)$.


\section{Influence as an outlying measure}


Define the binary vector
\begin{align}
\bz = \left[ z_1, \dots, z_N \right] \in \{ 0,1\}^N
\end{align}
whose role is to select subsets of $\cD$, where $z_i = 1$ implies that $\bp_i$ is selected and $z_i = 0$ means otherwise. Define $\bz_\cC$ as the binary vector which is all zero except at the positions where $i \in \cC$. A special case is
\begin{align}
\be_i = \bz_{\{ i\}},
\end{align}
i.e., the binary vector with all elements zero except the $i$-th one. Next, define
\begin{align}
\cC_\bz = \{ i \in \cD \mid z_i = 1 \},
\end{align}
i.e., the set of indices where the binary variables are $1$ in $\bz$.

Define \emph{feasibility test} $f: \{0,1 \}^{N} \mapsto \{0,1 \}$ where
\begin{align}\label{eq:feasibility}
f(\bz) = \begin{cases} 0 & \textrm{if}~g(\cC_\bz) \le \epsilon; \\ 1 & \textrm{otherwise}. \end{cases}
\end{align}
Intuitively, $\bz$ is feasible ($f(\bz)$ evaluates to $0$) if $\bz$ selects a consensus set of $\cD$. The \emph{influence} of a point $\bp_i$ is
\begin{align}\label{eq:influence}
\begin{aligned}
\alpha_i &= Pr\left[ f(\bz \oplus \be_i ) \ne f(\bz) \right]\\
&= \frac{1}{2^N} \left| \{ \bz \in \{0,1 \}^N \mid  f(\bz \oplus \be_i ) \ne f(\bz)  \} \right|.
\end{aligned}
\end{align}
In words, $\alpha_i$ is the probability of changing the feasibility of a subset $\bz$ by inserting/removing $\bp_i$ into/from $\bz$. Note that~\eqref{eq:influence} considers all $2^N$  instantiations of $\bz$.

The utility of $\alpha_i$ as a measure of outlyingness was proposed in~\cite{suter2020monotone}, as we further illustrate with examples below. Computing $\alpha_i$ will be discussed from Sec.~\ref{sec:classical} onwards.

Note that a basic requirement for $\alpha_i$ to be useful is that an appropriate $\epsilon$ can be input by the user. The prevalent usage of the consensus formulation~\eqref{eq:consensus} in computer vision~\cite{chin17} indicates that this is usually not a practical obstacle. 

\subsection{Examples}\label{sec:examples}

\paragraph{Line fitting}

The model $\cM$ is a line parametrised by $\bx \in \mathbb{R}^2$, and each $\bp_i$ has the form
\begin{align}
\bp_i = ( a_i, b_i ).
\end{align}
The residual function evaluates the ``vertical" distance
\begin{align}
r_i(\bx) = \left| [a_i,1]\bx - b_i \right|
\end{align}
from the line to $\bp_i$. The associated minimax problem~\eqref{eq:minimax} is a linear program~\cite{cheney66}, hence $g(\cC)$ can be evaluated efficiently. 

Fig.~\ref{fig:linefitting_data} plots a data instance $\cD$ with $N = 100$ points, while Fig.~\ref{fig:linefitting_influence} plots the sorted normalised influences of the points. A clear dichotomy between inliers and outliers can be observed in the influence.

\begin{figure}
	\centering
	\subfigure[Points on a plane.]{\includegraphics[width=0.45\textwidth]{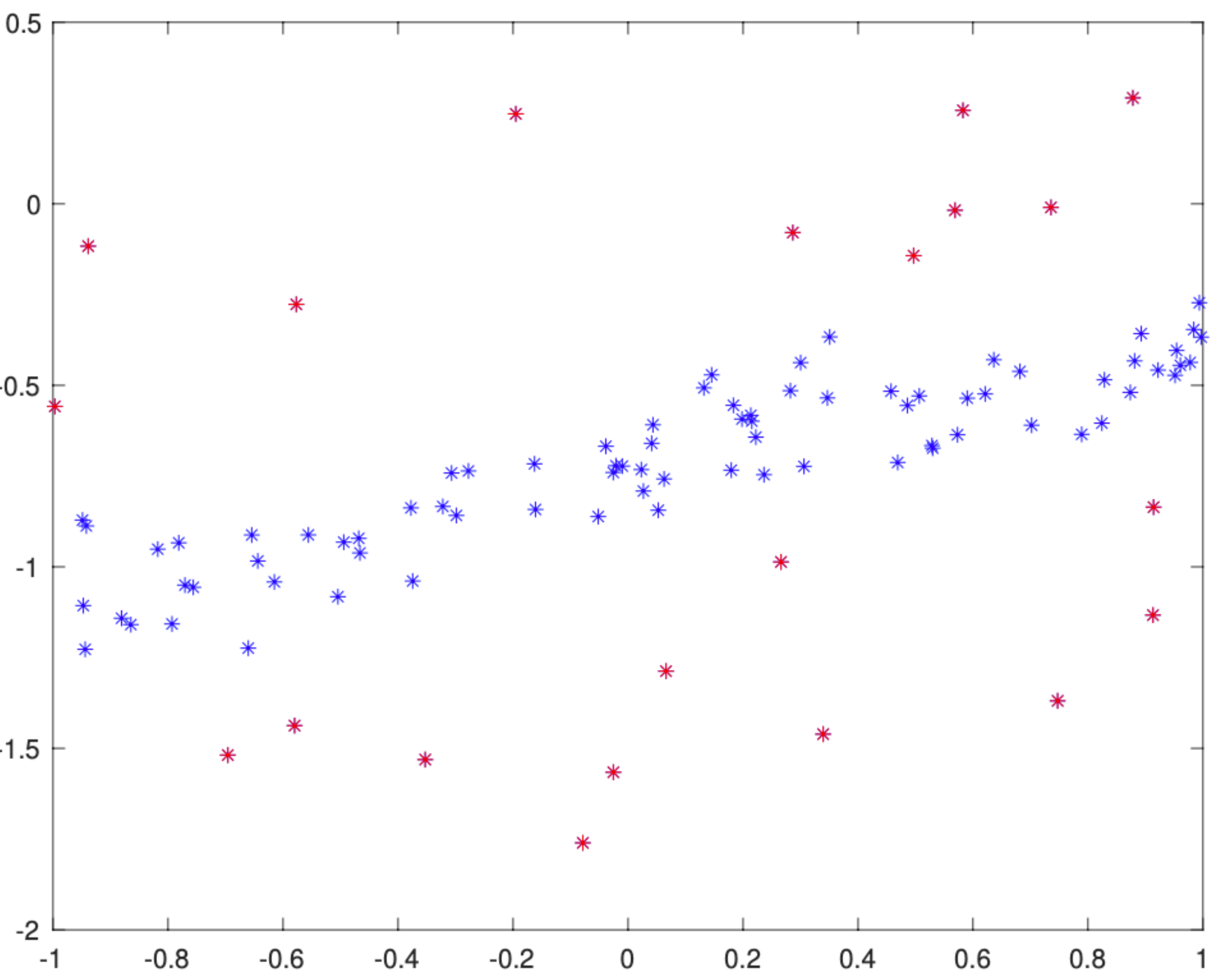}\label{fig:linefitting_data}}
	\subfigure[]{\includegraphics[width=0.49\textwidth]{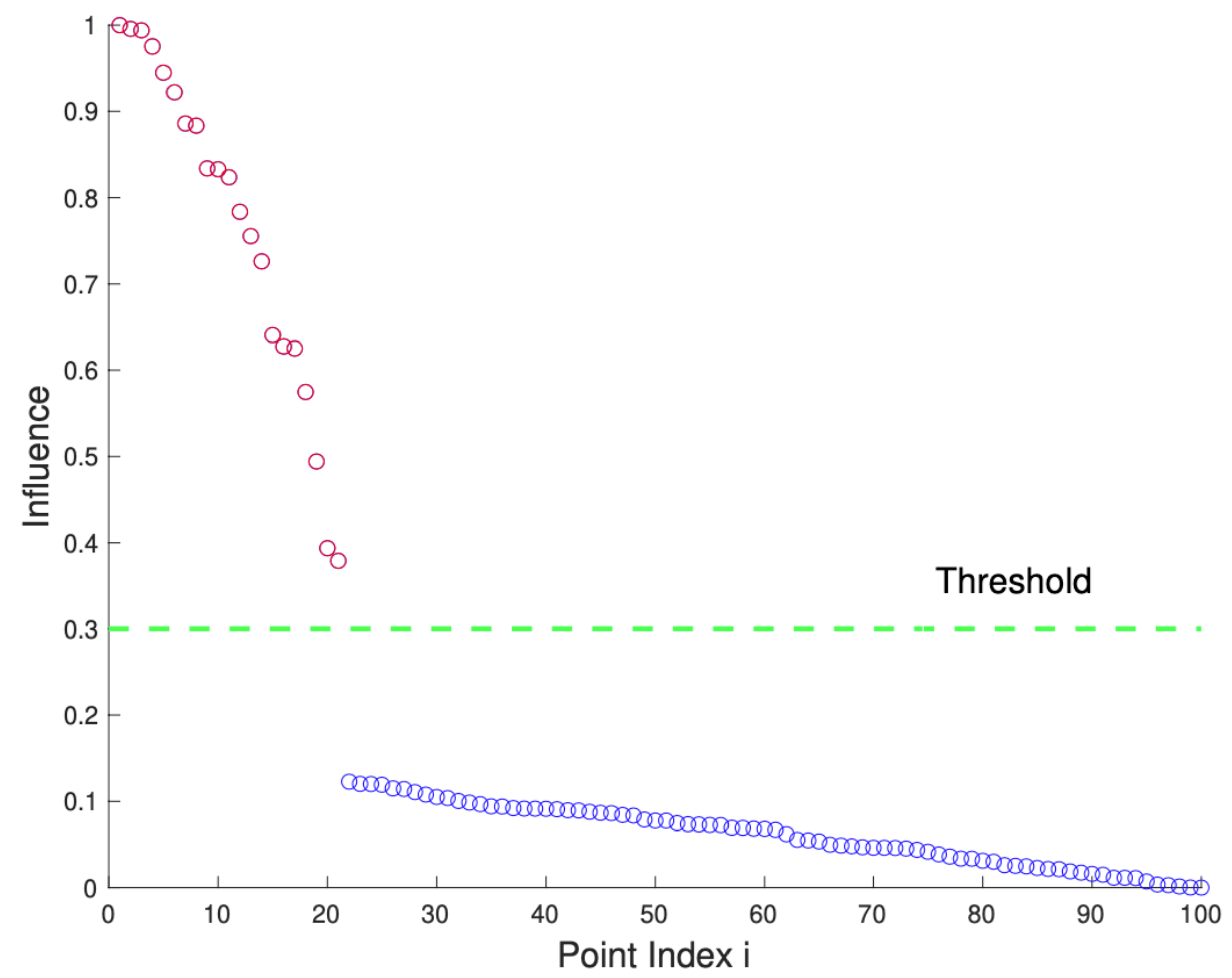}\label{fig:linefitting_influence}}
	\subfigure[Feature correspondences across multiple calibrated views.]{\includegraphics[width=0.45\textwidth]{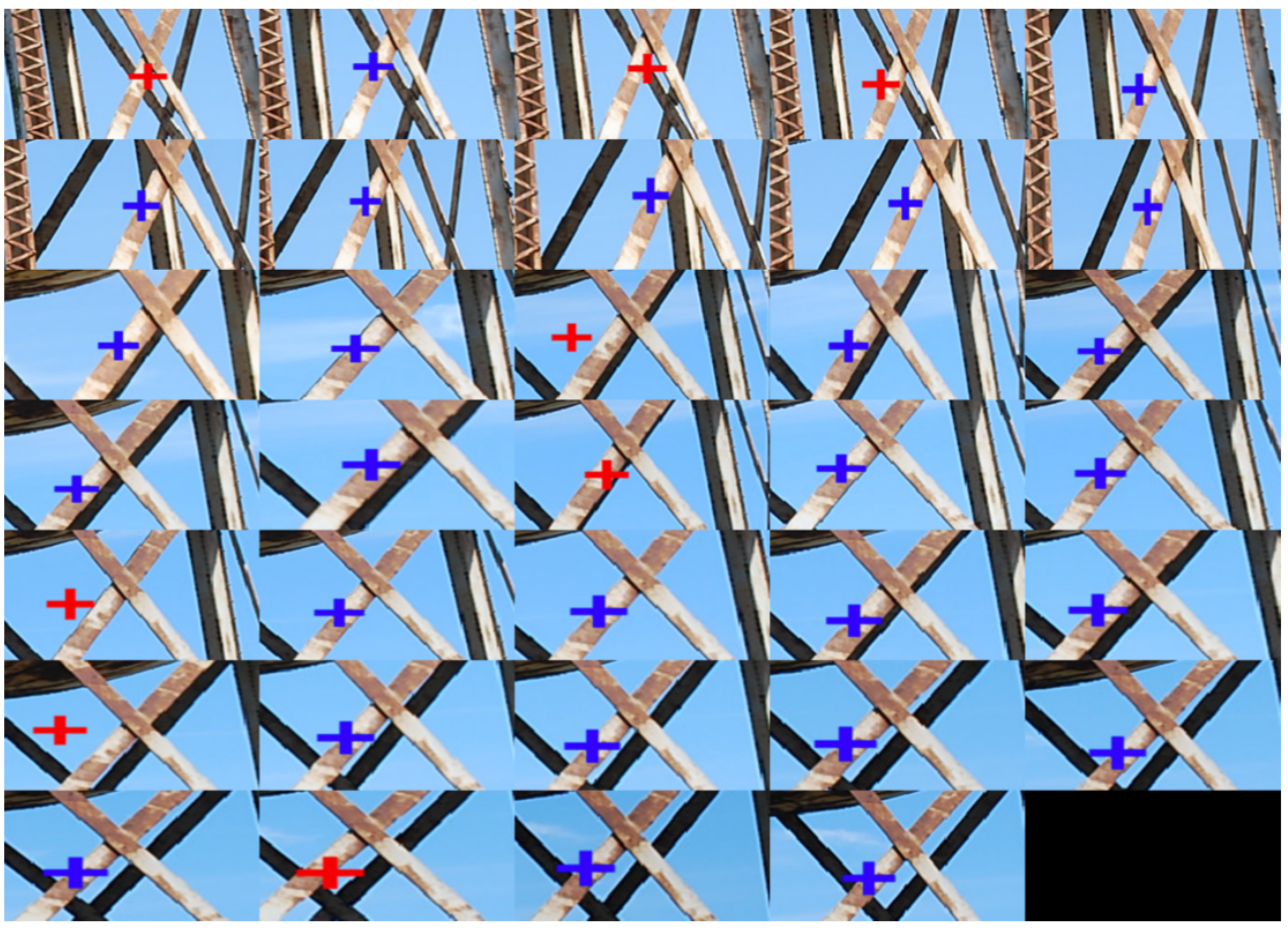}\label{fig:triangulation_data}}
	\subfigure[]{\includegraphics[width=0.49\textwidth]{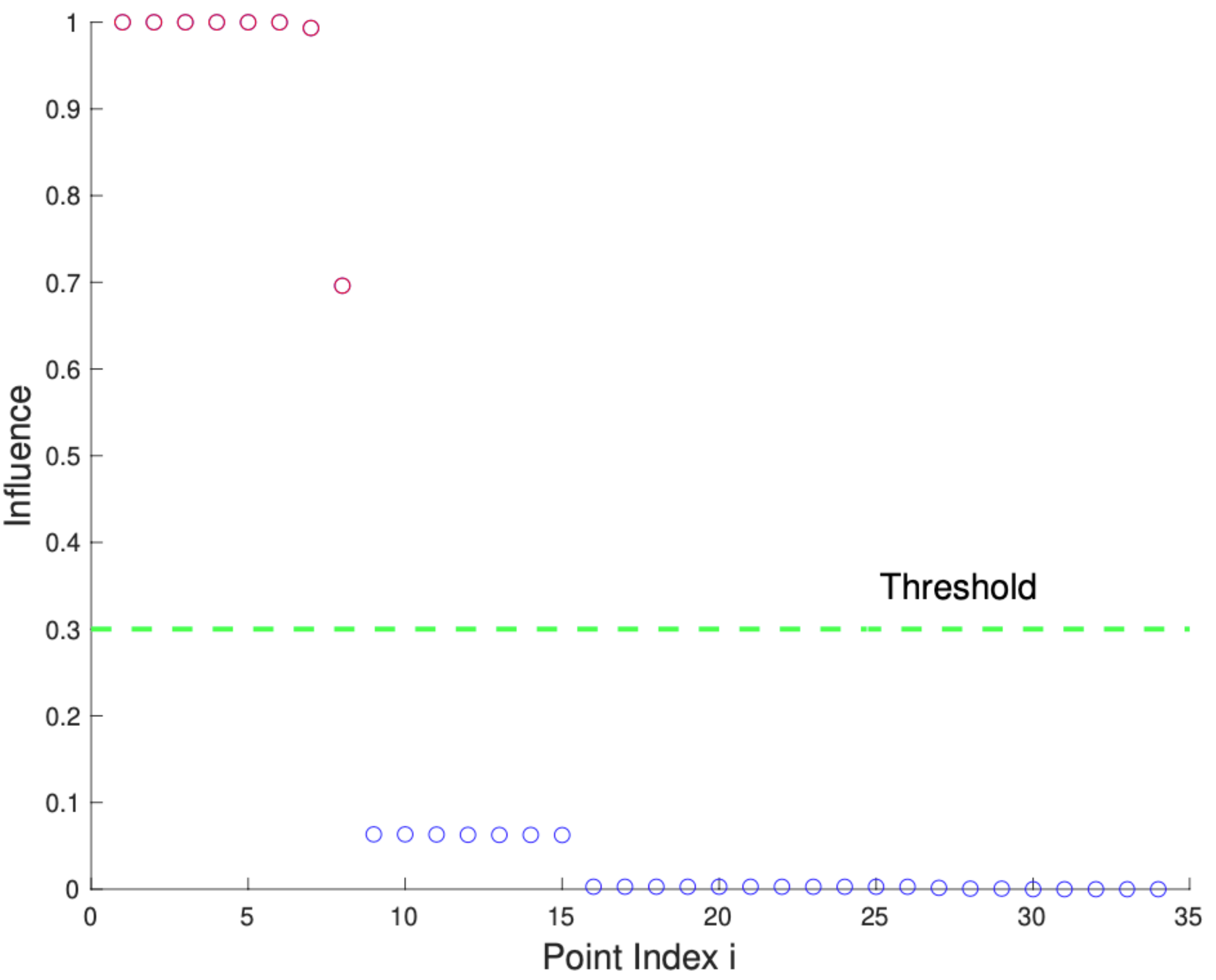}\label{fig:triangulation_influence}}
	\subfigure[Two-view feature correspondences.]{\includegraphics[width=0.45\textwidth]{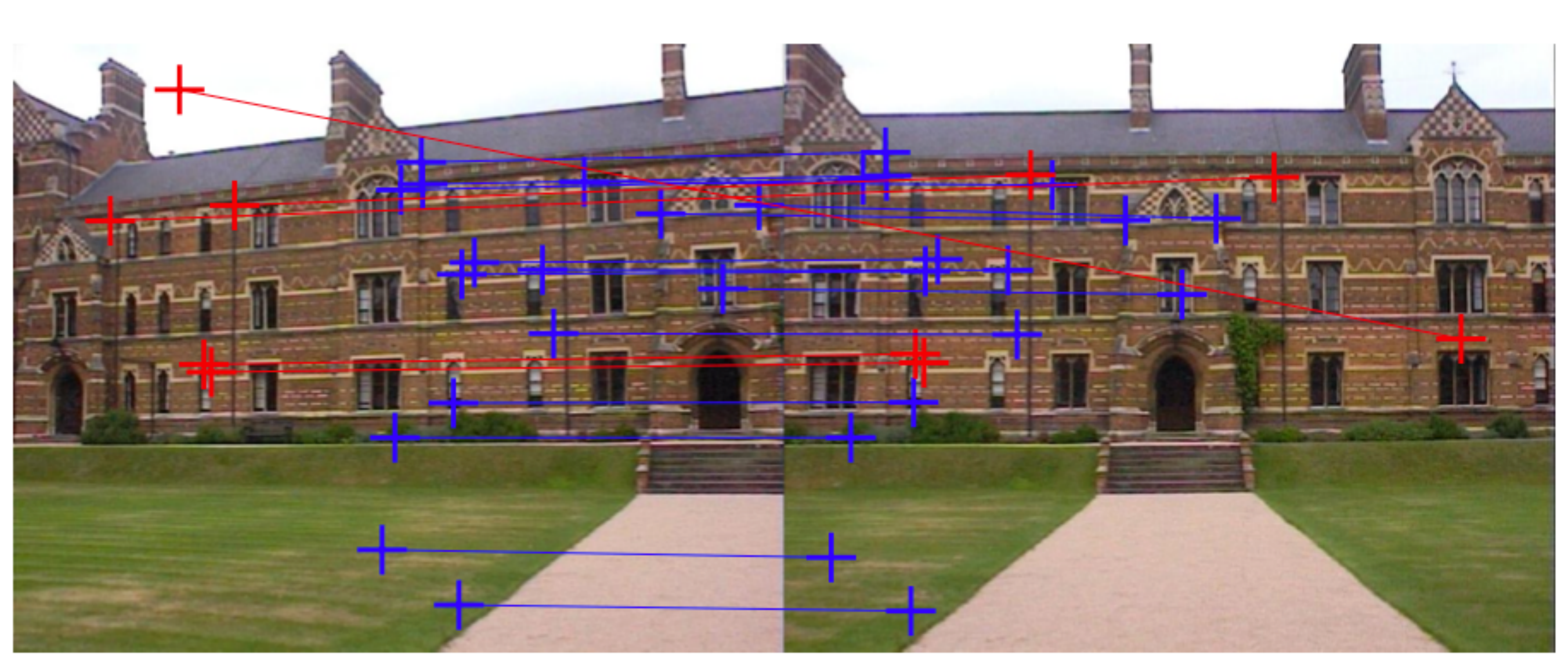}\label{fig:homography_data}}
	\subfigure[]{\includegraphics[width=0.49\textwidth]{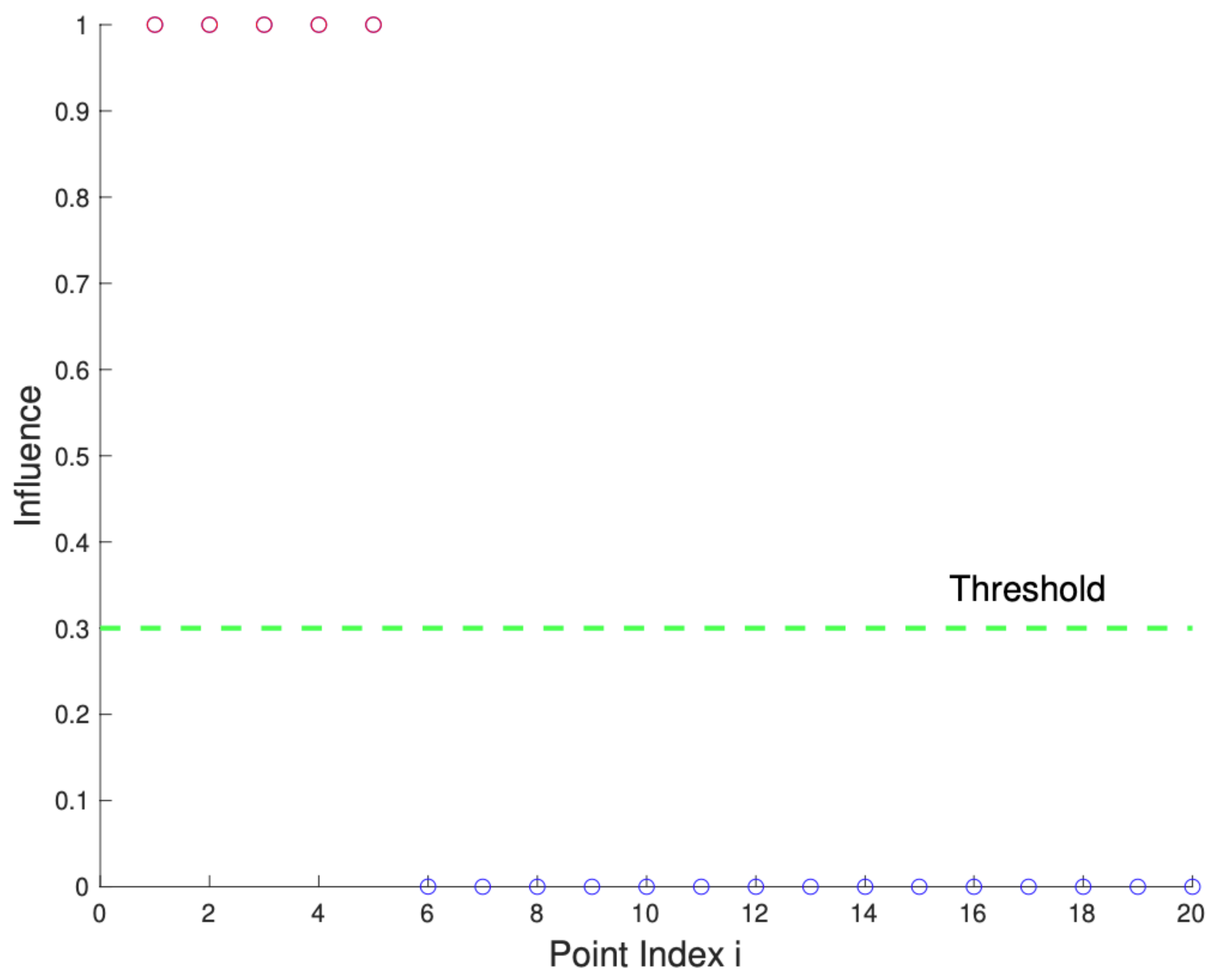}\label{fig:homography_influence}}
	\caption{Data instances $\cD$ with outliers (left column) and their normalised influences (right column). Row 1 shows a line fitting instance with $d = 2$ and $N = 100$; Row 2 shows a triangulation instance with $d = 3$ and $N = 34$; Row 3 shows a homography estimation instance with $d = 8$ and $N = 20$. In each result, the normalised influences were thresholded at 0.3 to separate the inliers (blue) and outliers (red).}
	\label{fig:exactcomputations}
\end{figure}

\paragraph{Multiple view triangulation}

Given observations $\cD$ of a 3D scene point $\cM$ in $N$ calibrated cameras, we wish to estimate the coordinates $\bx \in \mathbb{R}^3$ of $\cM$. The $i$-th camera matrix is $\bP_i \in \mathbb{R}^{3\times 4}$, and each data point $\bp_i$ has the form
\begin{align}
\bp_i = \left[ u_i, v_i \right]^T.
\end{align}
The residual function is the reprojection error
\begin{align}
r_i(\bx) = \left\| \bp_i - \frac{\bP_i^{1:2}\tilde{\bx}}{\bP^3_i\tilde{\bx}} \right\|_2,
\end{align}
where $\tilde{\bx} = [ \bx^T, 1 ]^T$, and $\bP^{1:2}_{i}$ and $\bP^{3}_{i}$ are respectively the first-two rows and third row of $\bP_i$. The reprojection error is quasiconvex in the region $\bP^3_i\tilde{\bx} > 0$~\cite{kahl08}, and the associated minimax problem~\eqref{eq:minimax} can be solved using generalised linear programming~\cite{eppstein05} or specialised routines such as bisection with SOCP feasibility tests; see~\cite{kahl08} for details.

Fig.~\ref{fig:triangulation_data} shows a triangulation instance $\cD$ with $N = 34$ image observations of the same scene point, while Fig.~\ref{fig:triangulation_influence} plots the sorted normalised influences of the data. Again, a clear dichotomy between inliers and outliers can be seen.

\paragraph{Homography estimation}

Given a set of feature matches $\cD$ across two images, we wish to estimate the homography $\cM$ that aligns the feature matches. The homography is parametrised by a homogeneous $3 \times 3$ matrix $\bH$ which is ``dehomogenised" by fixing one element (specifically, the bottom right element) to a constant of $1$, following~\cite{kahl08}. The remaining elements thus form the parameter vector $\bx \in \mathbb{R}^8$. Each data point $\bp_i$ contains matching image coordinates
\begin{align}
\bp_i =  ( \bu_i, \bv_i ).
\end{align}
The residual function is the transfer error
\begin{align}
r_i(\bx) = \frac{\| (\bH^{1:2} - \bv_i\bH_3)\tilde{\bu}_i \|_2}{\bH^3\tilde{\bu}_i},
\end{align}
where $\tilde{\bu}_i = [ \bu_i^T, 1 ]^T$, and $\bH^{1:2}_{i}$ and $\bH^{3}_{i}$ are respectively the first-two rows and third row of $\bH_i$. The transfer error is quasiconvex in the region $\bH^3\tilde{\bu}_i > 0$~\cite{kahl08}, which usually fits the case in real data; see~\cite{kahl08} for more details. As in the case of triangulation, the associated minimax problem~\eqref{eq:minimax} for the transfer error can be solved efficiently.

Fig.~\ref{fig:homography_data} shows a homography estimation instance $\cD$ with $N = 20$ feature correspondences, while Fig.~\ref{fig:homography_influence} plots the sorted normalised influences of the data. Again, a clear dichotomy between inliers and outliers can be observed.

\subsection{Robust fitting based on influence}\label{sec:robustfitting}

As noted in~\cite{suter2020monotone} and depicted above, the influence has a ``natural ability" to separate inliers and outliers; specifically, outliers tend to have higher influences than inliers. A basic robust fitting algorithm can be designed as follows:
\begin{enumerate}
	\item Compute influence $\{ \alpha_i \}^{N}_{i=1}$ for $\cD$ with a given $\epsilon$.
	\item Fit $\cM$ (e.g., using least squares) onto the subset of $\cD$ whose $\alpha_i \le \gamma$, where $\gamma$ is a predetermined threshold.
\end{enumerate}
See Fig.~\ref{fig:exactcomputations} and Fig.~\ref{fig:bighomography} for results of this simple algorithm. A more sophisticated usage of the influences could be to devise inlier probabilities based on influences and supply them to algorithms such as PROSAC~\cite{chum05} or USAC~\cite{raguram13}.

In the rest of this paper (Sec.~\ref{sec:classical} onwards), we will mainly be concerned with computing the influences $\{ \alpha_i \}^{N}_{i=1}$  as this is a major bottleneck in the robust fitting algorithm above. Note also that the algorithm assumes only a single structure in the data---for details on the behaviour of the influences under multiple structure data, see~\cite{suter2020monotone}.

\section{Classical algorithm}\label{sec:classical}

The naive method to compute influence $\alpha_i$ by enumerating $\bz$ is infeasible (although the enumeration technique was done in Fig.~\ref{fig:exactcomputations}, the instances there are low-dimensional $d$ or small in size $N$). A more practical solution is to sample a subset $\sfZ \subset \{0,1\}^N$ and approximate the influence as
\begin{align}\label{eq:approx1}
\hat{\alpha}_i = \frac{1}{|\sfZ|} \left| \{ \bz \in \sfZ \mid  f(\bz \oplus \be_i ) \ne f(\bz)  \} \right|.
\end{align}

Further simplification can be obtained by appealing to the existence of bases in the minimax problem~\eqref{eq:minimax} with quasiconvex residuals~\cite{eppstein05}. Specifically, $\cB \subseteq \cD$ is a basis if
\begin{align}
g(\cA) < g(\cB) \hspace{1em} \forall \cA \subset \cB.
\end{align}
Also, each $\cC \subseteq \cD$ contains a basis $\cB \subseteq \cC$ such that
\begin{align}
g(\cC) = g(\cB),
\end{align}
and, more importantly,
\begin{align}
g(\cC \cup \{i\}) = g(\cB \cup \{i\}).
\end{align}
Amenta et al.~\cite{amenta99} proved that the size of a basis (called the combinatorial dimension $k$) is at most $2d+1$. For the examples in Sec.~\ref{sec:examples} with residuals that are continuously shrinking, $k = d+1$. Small bases (usually $k \ll N$) enable quasiconvex problems to be solved efficiently~\cite{amenta99,matousek96,eppstein05}. In fact, some of the algorithms compute $g(\cC)$ by finding its basis $\cB$.

It is thus sufficient to sample $\sfZ$ from the set of all $k$-subsets of $\cD$. Algorithm~\ref{alg:classical} summarises the influence computation method for the quasiconvex case.

\begin{algorithm}
	\begin{algorithmic}[1]
		\REQUIRE $N$ input data points $\cD$, combinatorial dimension $k$, inlier threshold $\epsilon$, number of iterations $M$.
		\FOR{$m = 1,\dots,M$}
		\STATE $\bz^{[m]} \leftarrow$ Randomly choose $k$-tuple from $\cD$.
		\FOR{$i = 1,\dots,N$}
		\IF{$f(\bz^{[m]} \oplus \be_i) \ne f(\bz^{[m]})$}
		\STATE $X^{[m]}_i \leftarrow 1$.\label{step:xi1}
		\ELSE
		\STATE $X^{[m]}_i \leftarrow 0$.\label{step:xi0}
		\ENDIF
		\ENDFOR
		\ENDFOR
		\FOR{$i = 1,\dots,N$}
		\STATE $\hat{\alpha}_i \leftarrow \frac{1}{M}\sum_{m=1}^M X^{[m]}_i$.
		\ENDFOR
		\RETURN $\{ \hat{\alpha}_i \}^{N}_{i=1}$.
	\end{algorithmic}
	\caption{Classical algorithm to compute influence.}\label{alg:classical}
\end{algorithm}

\subsection{Analysis}\label{sec:analysis1}

For $N$ data points and a total of $M$ samples, the algorithm requires $\cO(NM)$ calls of $f$, i.e., $\cO(NM)$ instances of minimax~\eqref{eq:minimax} of size proportional to only $k$ and $d$ each.

How does the estimate $\hat{\alpha}_i$ deviate from the true influence $\alpha_i$? To answer this question, note that since $\sfZ$ are random samples of $\{ 0,1\}^N$, the $X^{[1]}_i, \dots, X^{[M]}_i$ calculated in Algorithm~\ref{alg:classical} (Steps~\ref{step:xi1} and~\ref{step:xi0}) are effectively i.i.d.~samples of
\begin{align}
X_i \sim \textrm{Bernoulli}(\alpha_i);
\end{align}
cf.~\eqref{eq:influence} and~\eqref{eq:approx1}. Further, the estimate $\hat{\alpha}_i$ is the empirical mean of the $M$ samples
\begin{align}
\hat{\alpha}_i = \frac{1}{M}\sum_{m=1}^M X^{[m]}_i.
\end{align}
By Hoeffding's inequality~\cite{wiki:hoeffding}, we have
\begin{align}\label{eq:converge1}
Pr(|\hat{\alpha}_i-\alpha_i| < \delta) > 1 -  2e^{-2M\delta^2},
\end{align}
where $\delta$ is a desired maximum deviation. In words,~\eqref{eq:converge1} states that as the number of samples $M$ increases, $\hat{\alpha}_i$ converges probabilistically to the true influence $\alpha_i$.

\subsection{Results}\label{sec:results}

\begin{figure}[t]\centering
	\subfigure[Result for the line fitting instance in~\ref{fig:linefitting_data}.]{\includegraphics[width=0.49\textwidth]{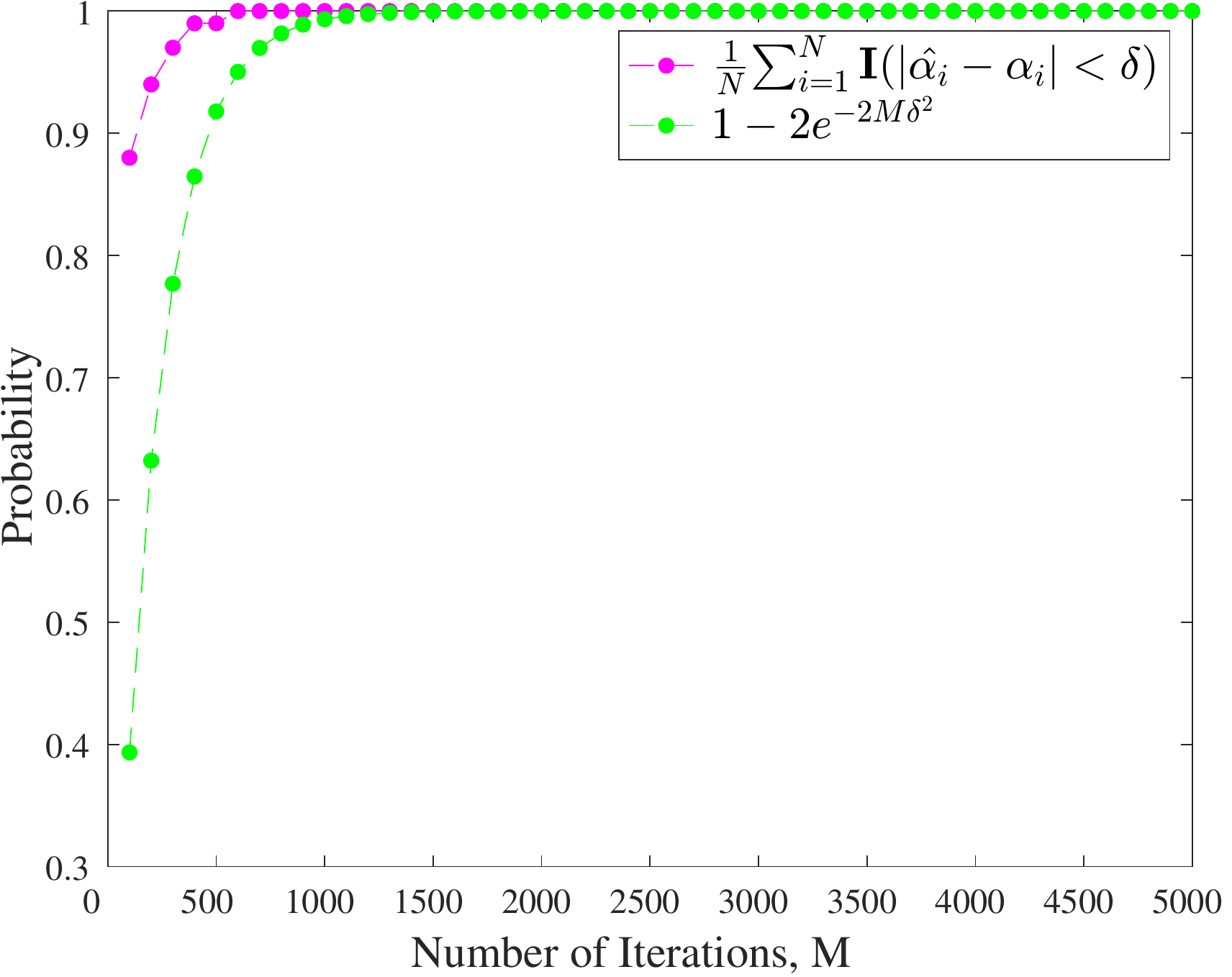}\label{fig:linefitting_err}}
	\hfill
	\subfigure[Result for the triangulation instance in~\ref{fig:triangulation_data}.]{\includegraphics[width=0.49\textwidth]{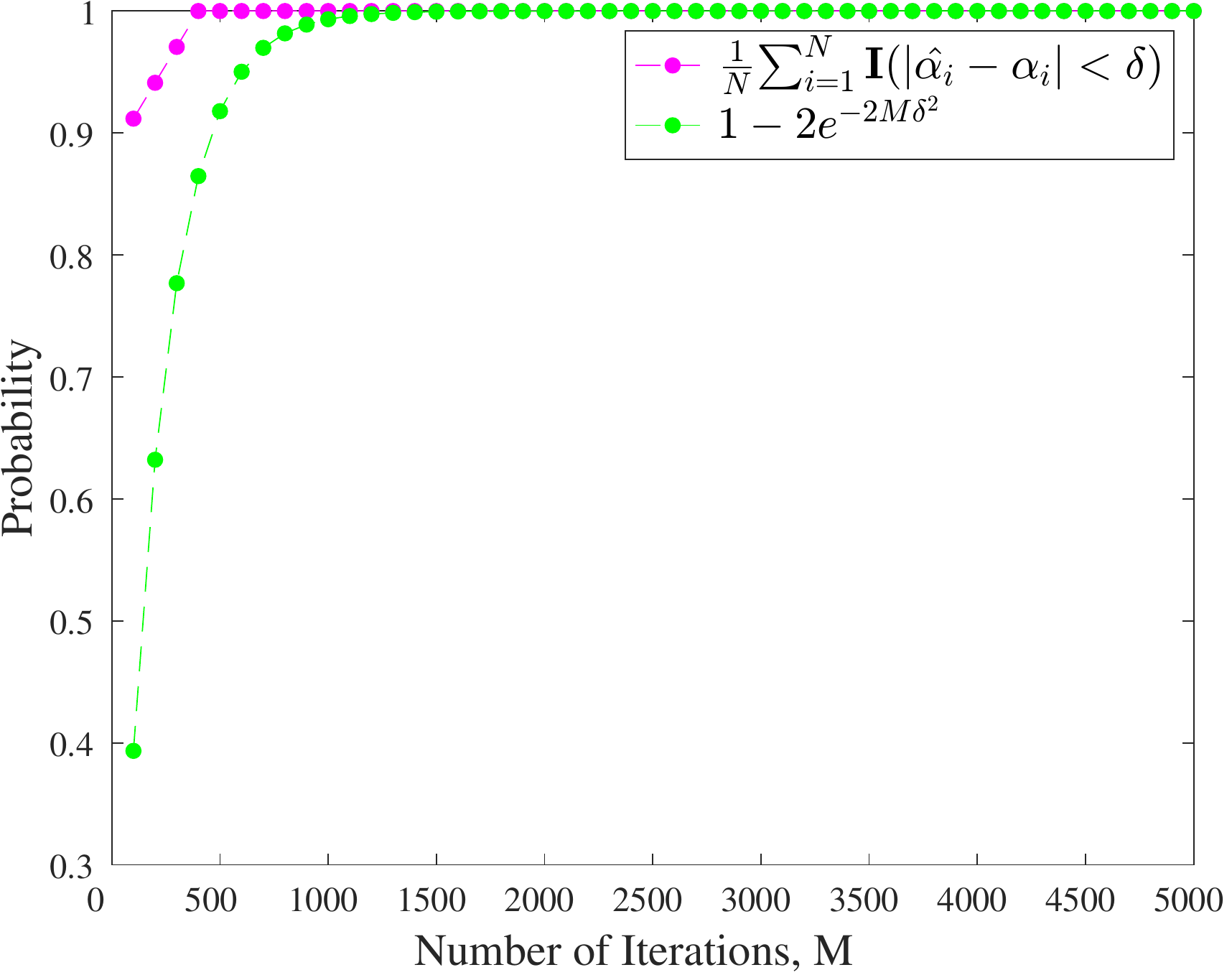}\label{fig:triangulation_err}}
	\subfigure[Result for the homography estimation instance in~\ref{fig:homography_data}.]{\includegraphics[width=0.49\textwidth]{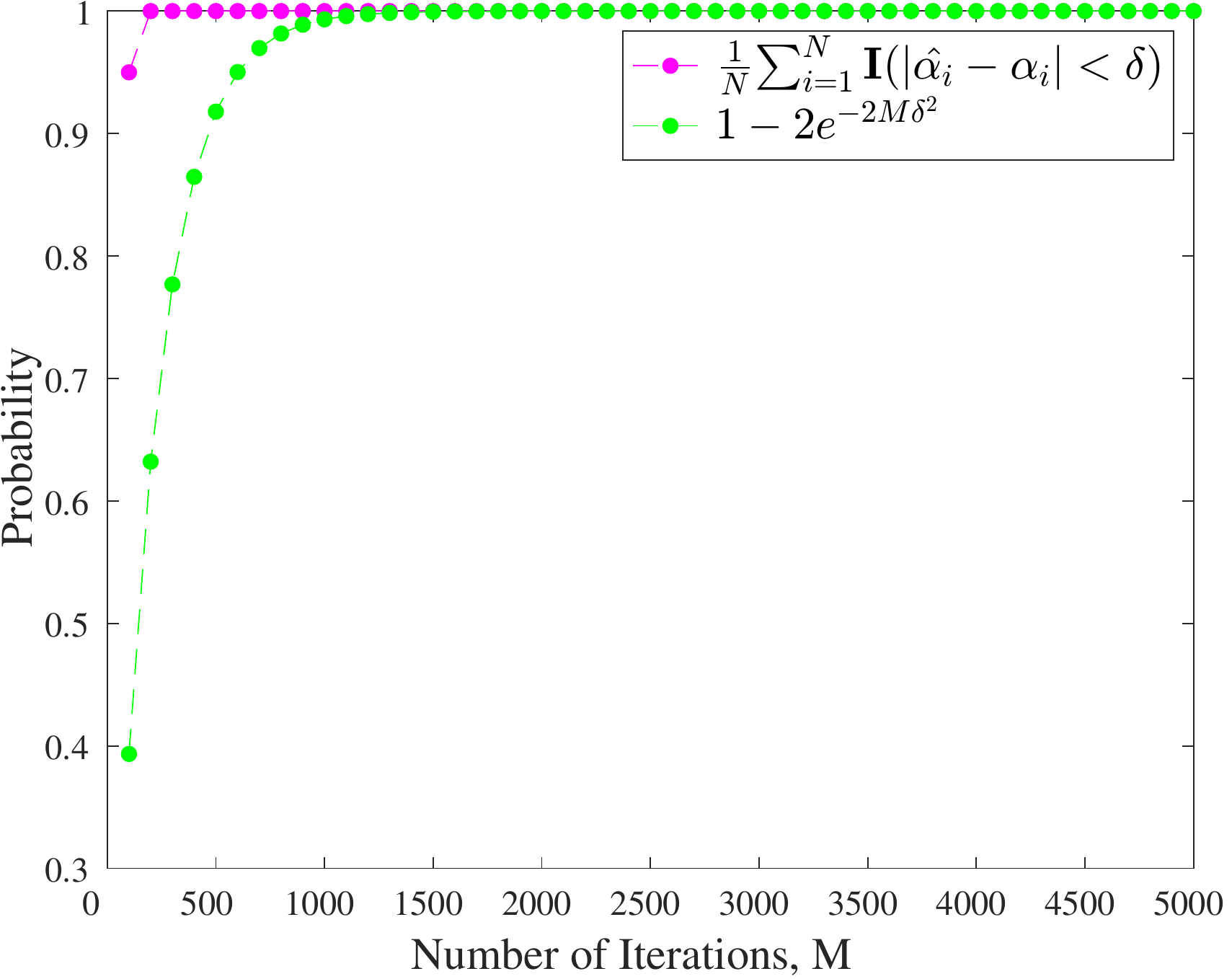}\label{fig:homography_err}}
	\caption{Comparing approximate influences from Algorithm~\ref{alg:classical} with the true influences~\eqref{eq:influence}, for the problem instances in Fig.~\ref{fig:exactcomputations}. The error of the approximation (magenta) is within the probabilistic bound (green). See Sec.~\ref{sec:results} on the error metric used.}
	\label{fig:approxcomputations}
\end{figure}

Fig.~\ref{fig:approxcomputations} illustrates the results of Algorithm~\ref{alg:classical} on the data in Fig.~\ref{fig:exactcomputations}. Specifically, for each input instance, we plot in Fig.~\ref{fig:approxcomputations} the proportion of $\{ \hat{\alpha}_i \}^{N}_{i=1}$ that are within distance $\delta = 0.05$ to the true $\{ \alpha_i \}^{N}_{i=1}$, i.e., 
\begin{align}
\frac{1}{N}\sum^{N}_{i=1} \mathbb{I}(| \hat{\alpha}_i - \alpha_i | < 0.05 ),
\end{align}
as a function of number of iterations $M$ in Algorithm~\ref{alg:classical}. The probabilistic lower bound $1 -  2e^{-2M\delta^2}$ is also plotted as a function of $M$. The convergence of the approximate influences is clearly as predicted by~\eqref{eq:converge1}.

Fig.~\ref{fig:bighomography} shows the approximate influences computed using Algorithm~\ref{alg:classical} on 3 larger input instances for homography estimation. Despite using a small number of iterations ($M \approx 800$), the inliers and outliers can be dichotomised well using the influences.

\begin{figure}
	\centering
	\subfigure[]{\includegraphics[width=0.45\textwidth]{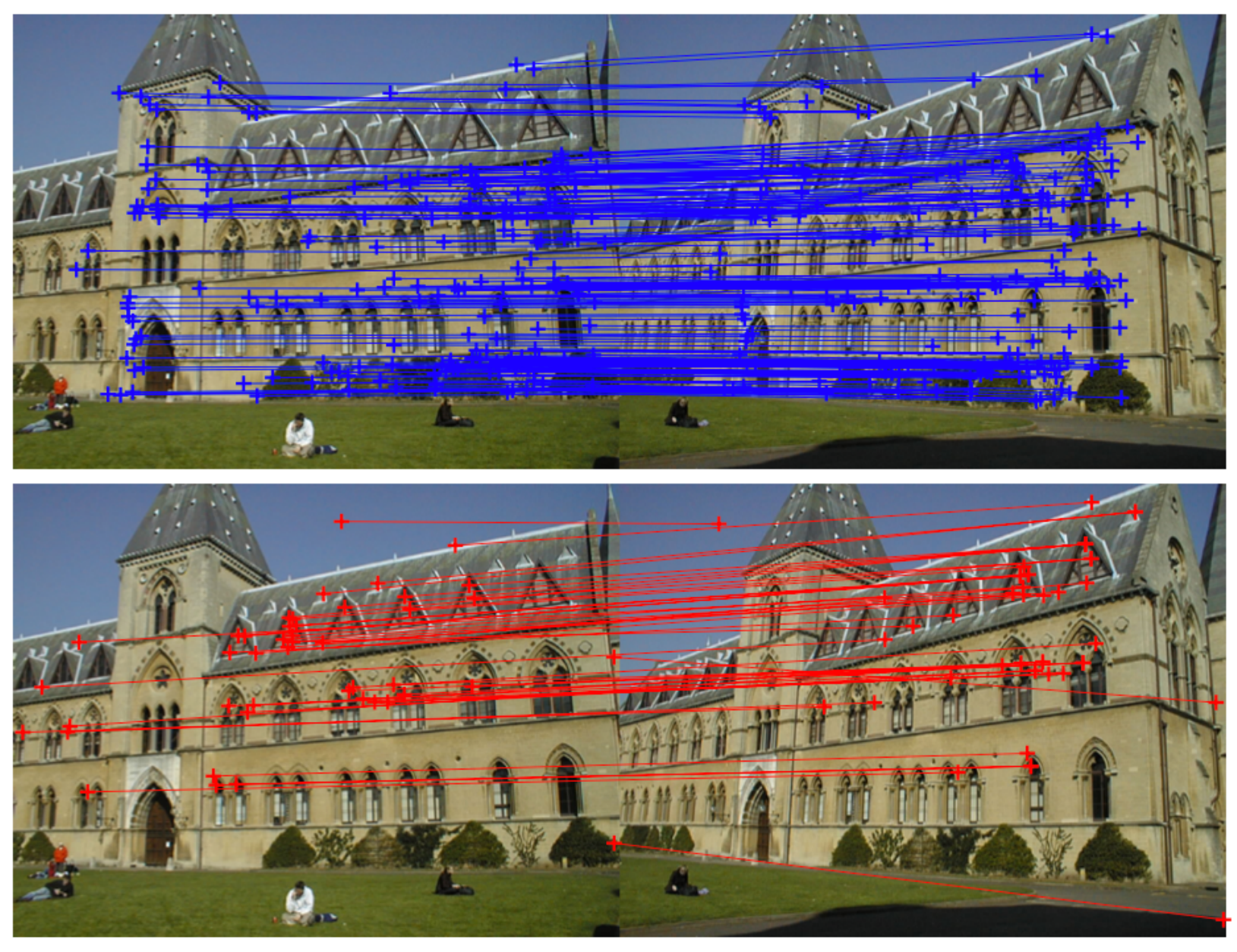}\label{fig:large_homography_data_measurement1}}		
	\subfigure[]{\includegraphics[width=0.50\textwidth]{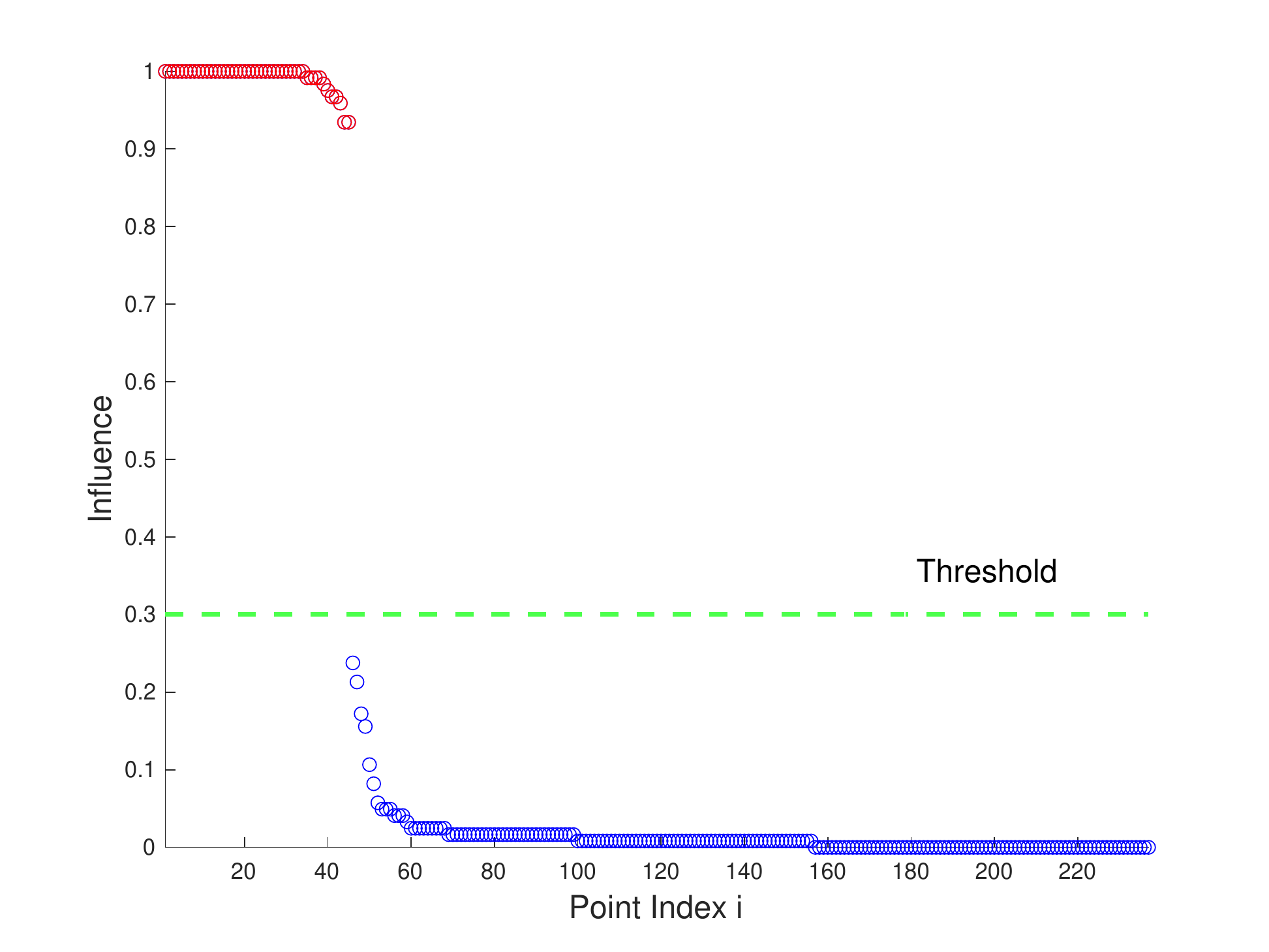}\label{fig:large_homography_data_influence_function1}}		
	
	\subfigure[]{\includegraphics[width=0.45\textwidth]{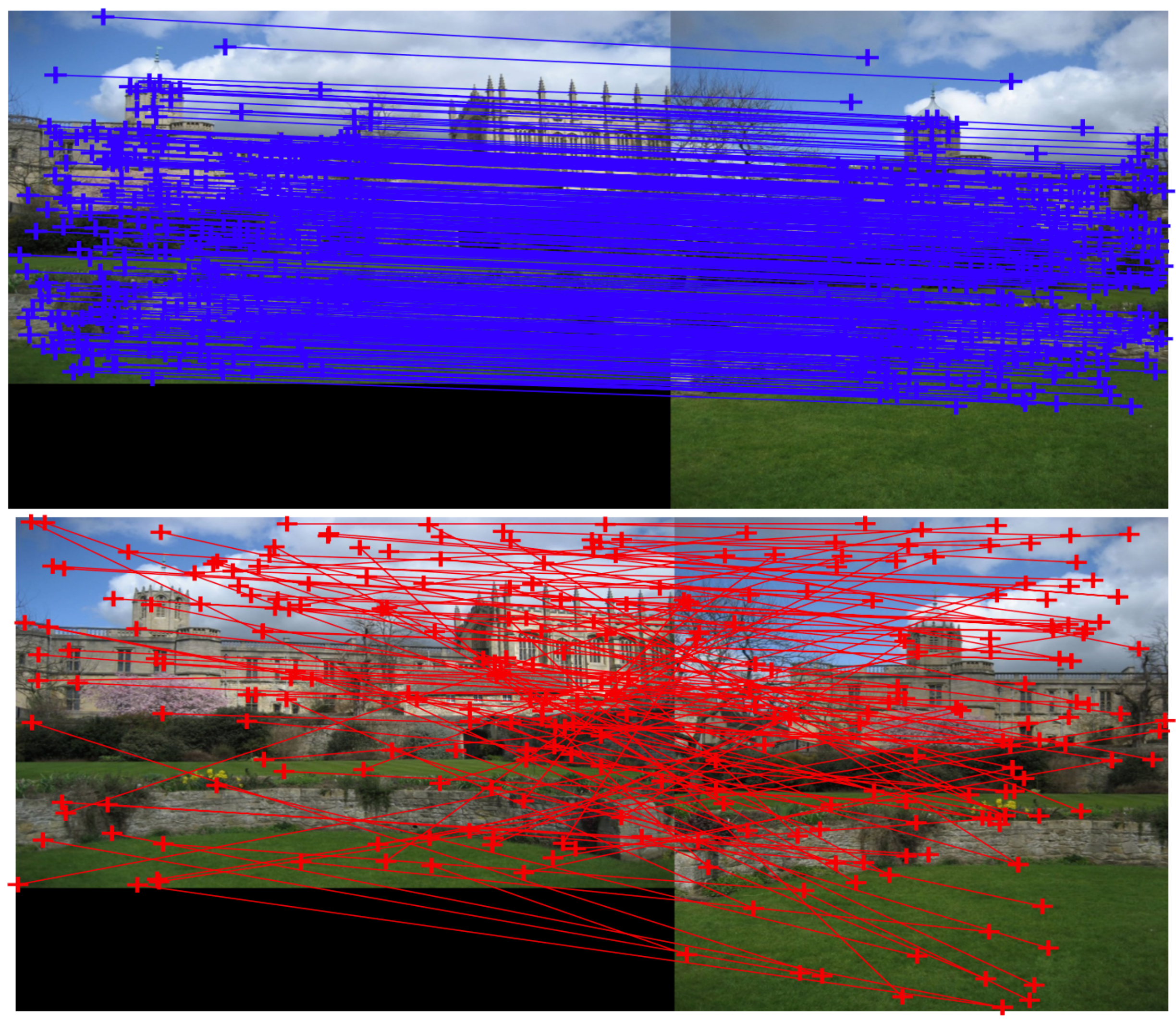}\label{fig:large_homography_data_measurement2}}	
	\subfigure[]{\includegraphics[width=0.50\textwidth]{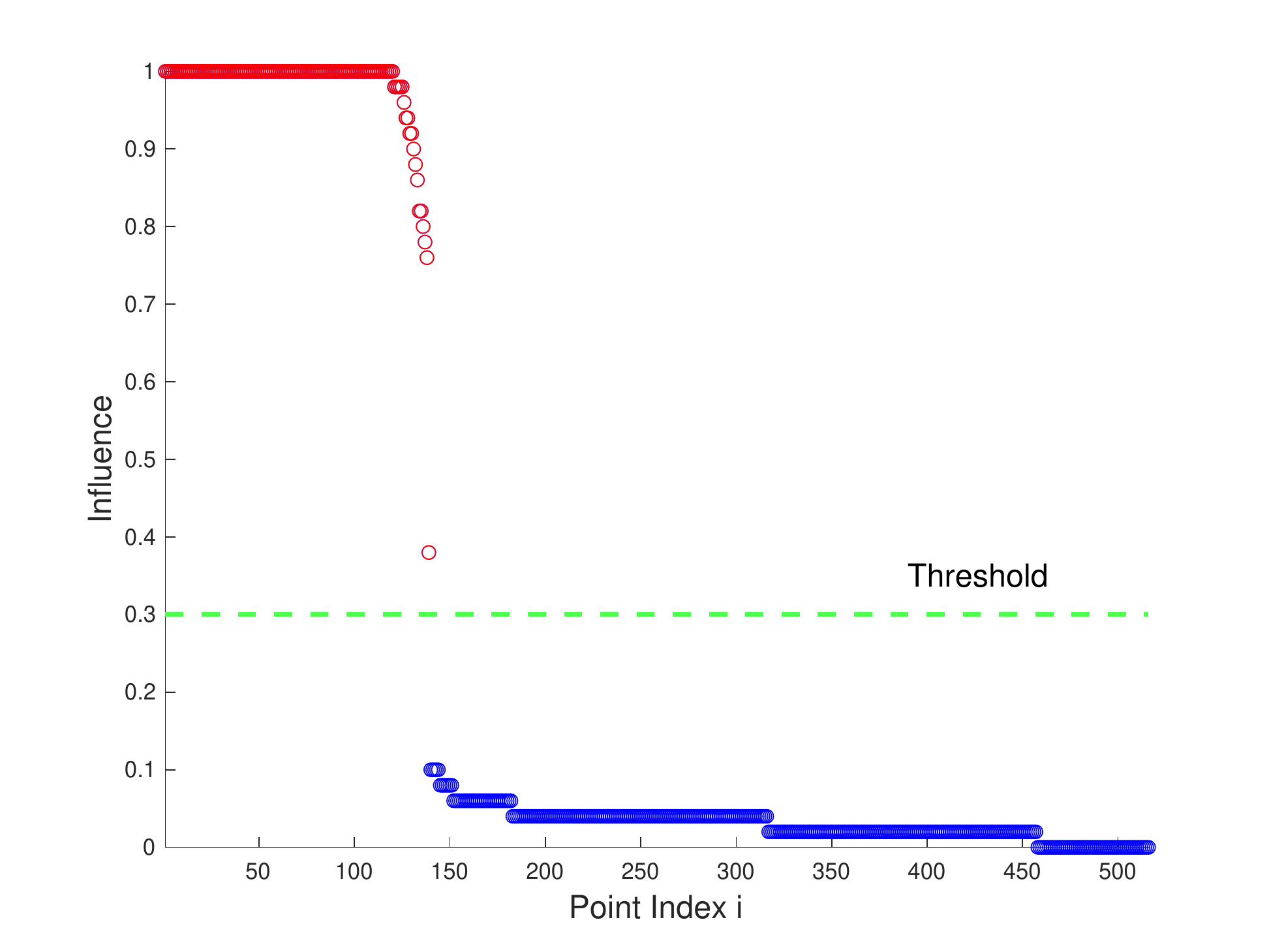}\label{fig:large_homography_data_influence_function2}}	
	
	\subfigure[]{\includegraphics[width=0.45\textwidth]{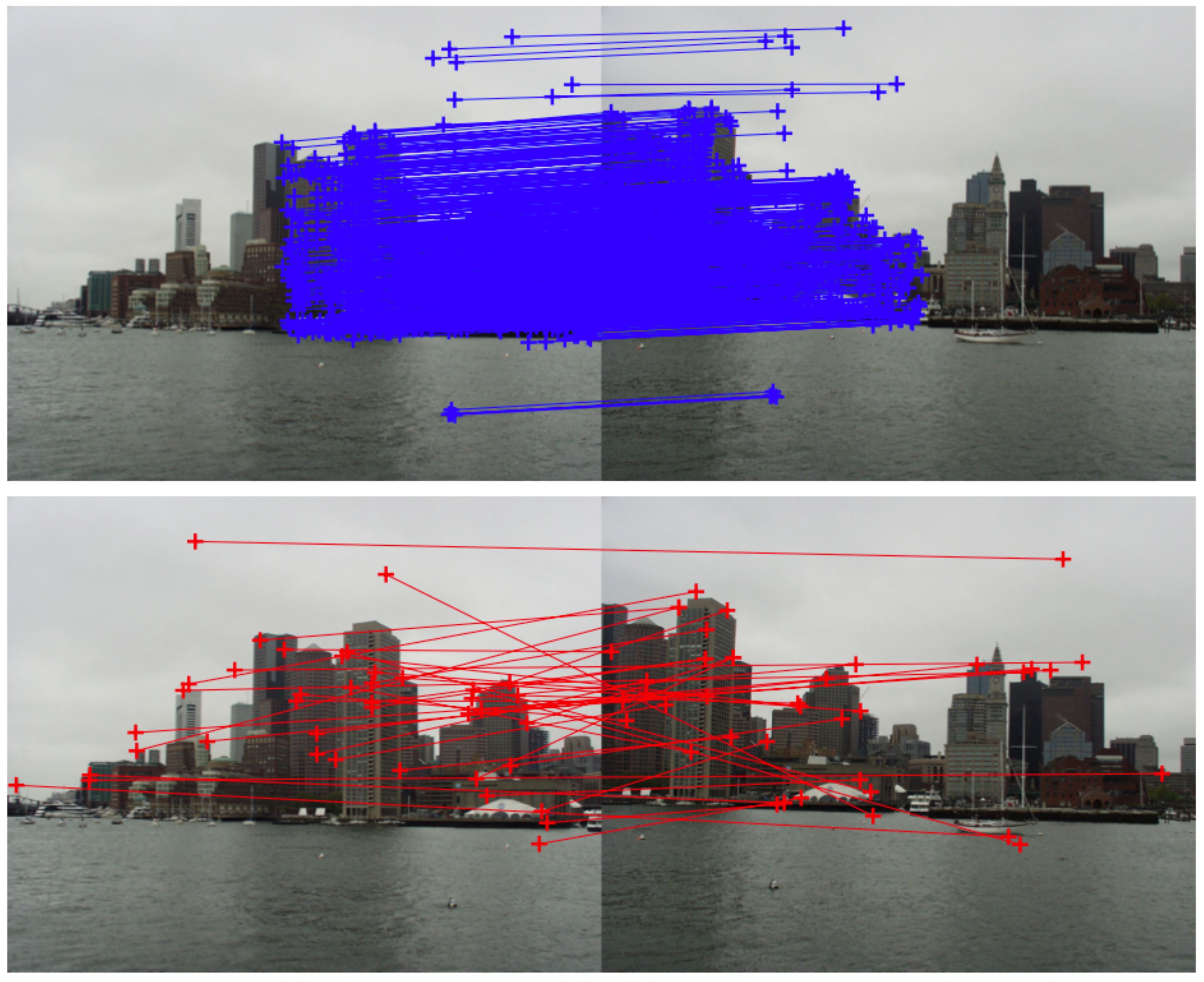}\label{fig:large_homography_data_measurement3}}	
	\subfigure[]{\includegraphics[width=0.50\textwidth]{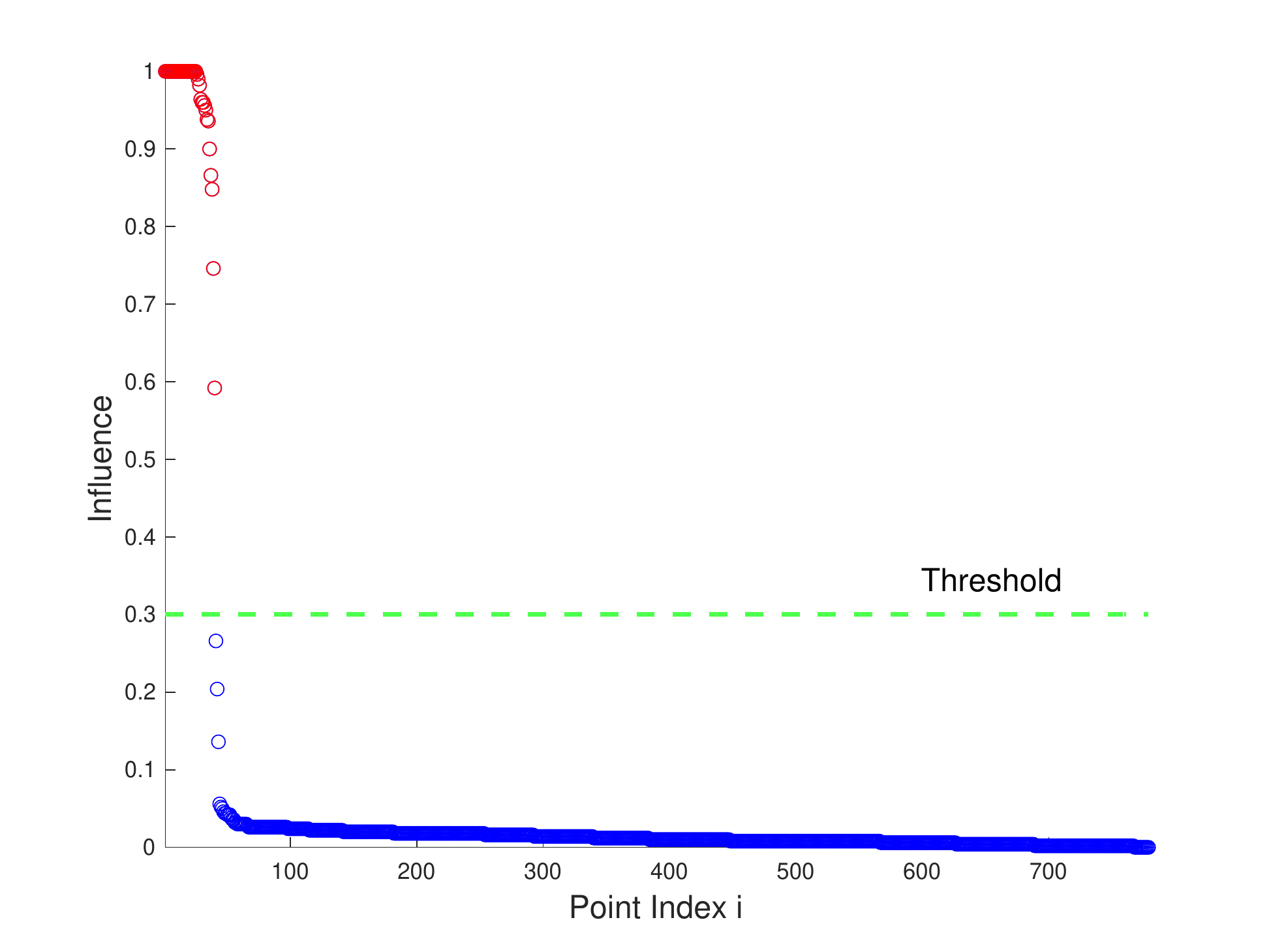}\label{fig:large_homography_data_influence_function3}}	
	\caption{Large homography estimation instances, separated into inliers (blue) and outliers (red) according to their normalised approximate influences (right column), which were computed using Algorithm~\ref{alg:classical}. Note that only about $M = 800$ iterations were used in Algorithm~\ref{alg:classical} to achieve these results. Row 1 shows an instance with $N=237$ correspondences; Row 2 shows an instance with $N=516$ correspondences; Row 3 shows an instance with $N=995$ correspondences.}
	\label{fig:bighomography}
\end{figure}

The runtimes of Algorithm 1 for the input data above are as follows:
\begin{center}
\begin{tabular}{|c|c|c|c|c|c|c|}
\hline
Input data (figure) & \ref{fig:linefitting_data} & \ref{fig:triangulation_data} & \ref{fig:homography_data} & \ref{fig:large_homography_data_measurement1} & \ref{fig:large_homography_data_measurement2} & \ref{fig:large_homography_data_measurement3} \\ 
\hline
Iterations ($M$) & 5,000 & 5,000 & 5,000 & 800 & 800 & 800 \\
\hline
Runtime (s) & 609 & 4,921 & 8,085 & 2,199 & 5,518 & 14,080 \\
\hline
\end{tabular}
\end{center}
The experiments were conducted in MATLAB on a standard desktop using unoptimised code, e.g., using \texttt{fmincon} to evaluate $f$ instead of more specialised routines.

\section{Quantum algorithm}\label{sec:quantum}

We describe a quantum version of Algorithm~\ref{alg:classical} for influence computation and investigate the speed-up provided.

\subsection{Quantum circuit}

We use the Bernstein-Vazirani (BV) circuit~\cite{bernstein97} originally designed to solve linear Boolean functions; see Fig.~\ref{fig:bv}.

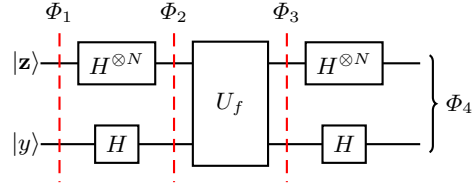
\begin{figure}[h]\centering
	\begin{quantikz}
		\ket{\bf{z}}\slice{$\Phi_1$} & \gate[wires=1][1cm]{H^{\otimes N}}\slice{$\Phi_2$} & \gate[wires=2][1cm]{U_f}\slice{$\Phi_3$} & \gate[wires=1][1cm]{H^{\otimes N}} & \qw \rstick[wires=2]{$\Phi_4$} \\
		\ket{y}                    &   \gate{H}                                                            &                                         &             \gate{H}                                                    & \qw 
	\end{quantikz}
	\caption{Quantum circuit for influence computation.}\label{fig:bv}
\end{figure}

For our application, our circuit builds an $(N+1)$-qubit system, where $N$ is the number of points $\cD$. At the input stage, the system contains quantum registers
\begin{align}
\ket{\bz} \otimes \ket{y} = \ket{\bz}\ket{y},
\end{align}
where, as before, $\bz$ is a binary vector indicating selection of points in $\cD$, and $y$ is a dummy input.

The Boolean function $f$~\eqref{eq:feasibility} and the underlying data $\cD$ are implemented in the quantum oracle $U_f$, where
\begin{align}
U_f\ket{\bz}\ket{y} = \ket{\bz}\ket{y\otimes f(\bz)}
\end{align}
and $\oplus$ is bit-wise XOR. Recall that by considering only quasiconvex residual functions (Sec.~\ref{sec:prelim}), $f$ is classically solvable in polynomial time, thus its quantum equivalent $U_f$ will also have an efficient implementation (requiring polynomial number of quantum gates)~\cite[Sec.~3.25]{nielsen10}. Following the analysis of the well-known quantum algorithms (e.g., Grover's search, Shor's factorisation algorithm), we will mainly be interested in the number of times we need to invoke $U_f$ (i.e., the query complexity of the algorithm~\cite{ambainis18}) and not the implementation details of $U_f$ (Sec.~\ref{sec:qanal}).

\subsection{Quantum operations}

Our usage of BV follows that of~\cite{floess13,li15} (for basics of quantum operations, see~\cite[Chapter 5]{rieffel14}). We initialise with $\ket{\bz} = \ket{\mathbf{0}}$ and $\ket{y} = \ket{1}$ thus
\begin{align}
\Phi_1 = \ket{\bzero}\ket{1}.
\end{align}
The next operation consists of $N+1$ Hadamard gates $H^{\otimes (N+1)}$; the behaviour of $n$ Hadamard gates is as follows
\begin{align}
H^{\otimes n}\ket{\bq} = \frac{1}{\sqrt{2^n}} \sum_{\bt \in \{0,1 \}^n} (-1)^{\bq \cdot \bt}\ket{\bt},
\end{align}
hence
\begin{align}
\Phi_2 &= H^{\otimes(N+1)}\Phi_1\\
&= \frac{1}{\sqrt{2^N}}\sum_{\bt \in \{0,1\}^N} \ket{\bt}\frac{\ket{0}-\ket{1}}{\sqrt{2}}.
\end{align}
Applying $U_f$, we have
\begin{align}
\Phi_3 &= U_f \Phi_2\\
&= \frac{1}{\sqrt{2^N}} \sum_{\bt \in \{0,1 \}^N} (-1)^{f(\bt)}\ket{\bt} \frac{\ket{0}-\ket{1}}{\sqrt{2}}.
\end{align}
Applying the Hadamard gates $H^{\otimes (N+1)}$ again,
\begin{align}
\Phi_4 &= H^{\otimes (N+1)}\Phi_3\\
&= \frac{1}{2^N} \sum_{\bs \in \{0,1\}^N} \sum_{\bt \in \{0,1 \}^N} (-1)^{f(\bt)+\bs\cdot\bt}\ket{\bs}\ket{1}.
\end{align}
Focussing on the top-$N$ qubits in $\Phi_4$, we have
\begin{align}
\sum_{\bs \in \{0,1\}^N} I(\bs) \ket{\bs},
\end{align}
where
\begin{align}
I(\bs) := \sum_{\bt \in \{0,1 \}^N} (-1)^{f(\bt)+\bs\cdot\bt}.
\end{align}
The significance of this result is as follows.

\begin{theorem}
	Let $\bs = [s_1,\dots,s_N] \in \{0,1 \}^N$. Then
	\begin{align}
	\alpha_i = \sum_{s_i = 1} I(\bs)^2.
	\end{align}
\end{theorem}
\begin{proof}
	See~\cite[Sec.~3]{li15}.
\end{proof}

The theorem shows that the influences $\{ \alpha_i \}^{N}_{i=1}$ are ``direct outputs" of the quantum algorithm. However, physical laws permit us to access the information indirectly via quantum measurements only~\cite[Chapter 4]{rieffel14}. Namely, if we measure in the standard basis, we get a realisation $\bs$ with probability $I(\bs)^2$. The probability of getting $s_i = 1$ is 
\begin{align}
Pr(s_i = 1) = \sum_{s_i = 1} I(\bs)^2 = \alpha_i.
\end{align}

Note that the above steps involve only one ``call" to $U_f$. However, as soon as $\Phi_4$ is measured, the quantum state collapses and the encoded probabilities vanish.

\subsection{The algorithm}

Based on the setup above, running the BV algorithm \emph{once} provides a \emph{single} observation of \emph{all} $\{ \alpha_i \}^{N}_{i=1}$. This provides a basis for a quantum version of the classical Algorithm~\ref{alg:classical}; see Algorithm~\ref{alg:quantum}. The algorithm runs the BV algorithm $M$ times, each time terminating with a measurement of $\bs$, to produce $M$ realisations $\bs^{[1]}, \dots, \bs^{[M]}$. Approximate estimates $\{ \hat{\alpha}_i \}^{N}_{i=1}$ of the influences are then obtained by collating the results of the quantum measurements.

\begin{algorithm}
	\begin{algorithmic}[1]
		\REQUIRE $N$ input data points $\cD$, inlier threshold $\epsilon$, number of iterations $M$.
		\FOR{$m = 1,\dots,M$}
		\STATE $\bs^{[m]} \leftarrow$ Run BV algorithm with $\cD$ and $\epsilon$ and measure top-$N$ qubits in standard basis.
		\ENDFOR
		\FOR{$i = 1,\dots,N$}
		\STATE $\hat{\alpha}_i \leftarrow \frac{1}{M}\sum_{m=1}^M s^{[m]}_i$.
		\ENDFOR
		\RETURN $\{ \hat{\alpha}_i \}^{N}_{i=1}$.
	\end{algorithmic}
	\caption{Quantum algorithm to compute influence~\cite{floess13,li15}.}\label{alg:quantum}
\end{algorithm}

\subsection{Analysis}\label{sec:qanal}

A clear difference between Algorithms~\ref{alg:classical} and~\ref{alg:quantum} is the lack of an ``inner loop" in the latter. Moreover, in each (main) iteration of the quantum algorithm, the BV algorithm is executed only once; hence, the Boolean function $f$ is also called just once in each iteration. The overall query complexity of Algorithm~\ref{alg:quantum} is thus $\mathcal{O}(M)$, which is a speed-up over Algorithm~\ref{alg:classical} by a factor of $N$. For example, in the case of the homography estimation instance in Fig.~\ref{fig:bighomography}, this represents a sizeable speed-up factor of $516$.

In some sense, the BV algorithm computes the influences exactly in one invocation of $f$; however, limitations placed by nature allows us to ``access" the results using probabilistic measurements only, thus delivering only approximate solutions. Thankfully, the same arguments in Sec.~\ref{sec:analysis1} can be made for Algorithm~\ref{alg:quantum} to yield the probabilistic error bound~\eqref{eq:converge1} for the results of the quantum version.

As alluded to in Sec.~\ref{sec:quantum}, the computational gain is based on analysing the query complexity of the algorithm~\cite{ambainis18}, i.e., the number of times $U_f$ needs to be invoked, which in turn rests on the knowledge that any polynomial-time classical algorithm can be implemented as a quantum function $f$ efficiently, i.e., with a polynomial number of gates (see~\cite[Chap.~3.25]{nielsen10} and~\cite[Chap.~6]{rieffel14}). In short, the computational analysis presented is consistent with the literature on the analysis of quantum algorithms.

\section{Conclusions and future work}

We proposed one of the first quantum robust fitting algorithms and established its practical usefulness in the computer vision setting. Future work includes devising quantum robust fitting algorithms that have better speed-up factors and tighter approximation bounds. Implementing the algorithm on a quantum computer will also be pursued.


\bibliographystyle{splncs}
\bibliography{quantum_robust_fitting_final}

\end{document}